\documentclass[11pt,a4paper]{article}

\usepackage{amsmath,amssymb,amsthm}
\usepackage{graphicx}
\usepackage{booktabs}
\usepackage{hyperref}
\usepackage{float}
\usepackage{algorithm}
\usepackage{algorithmic}
\usepackage{natbib}
\usepackage{geometry}
\usepackage{subcaption}
\usepackage{multirow}
\graphicspath{{figures/}}

\newtheorem{theorem}{Theorem}
\newtheorem{proposition}{Proposition}

\newtheorem{corollary}{Corollary}
\theoremstyle{definition}
\newtheorem{definition}{Definition}
\newtheorem{remark}{Remark}

\newcommand{\R}{\mathbb{R}}
\newcommand{\E}{\mathbb{E}}

\title{Lévy-Flow Models: Heavy-Tail-Aware Normalizing Flows\\for Financial Risk Management}

\author{
  R. Drissi\\
  \texttt{rdrissi@gmail.com}
}

\date{March 2026}

\begin{document}

\maketitle

\begin{abstract}
Standard normalizing flows use Gaussian base distributions, which systematically underestimate tail risk in financial applications. We introduce \textbf{Lévy-Flows}, a novel class of normalizing flows that replace the Gaussian base with Lévy process-based distributions---specifically Variance Gamma (VG) and Normal-Inverse Gaussian (NIG). These distributions naturally capture heavy tails while preserving exact density evaluation and efficient reparameterized sampling.

For bases with regularly varying (power-law) tails, we prove that the tail index is preserved under asymptotically linear flow transformations. For the semi-heavy-tailed VG and NIG bases used in practice, we show that the identity-tail structure of Neural Spline Flows preserves the base distribution's tail shape exactly outside the spline region. Experiments on S\&P 500 daily returns (2000--2025) and additional assets show that Lévy-Flows substantially improve both density estimation and risk calibration: VG-based flows reduce test negative log-likelihood by 69\% relative to Gaussian flows and achieve exact 95\% VaR calibration (Kupiec $p = 1.00$), while NIG-based flows provide the most accurate Expected Shortfall estimates (1.6\% underestimation vs.\ 10.4\% for Gaussian). Fixed-parameter Student-t flows do not materially improve over Gaussian baselines in density estimation, suggesting that the Lévy parametric structure---not simply heavier tails---drives the gains. Different Lévy bases may be preferable depending on whether the target is density fit, VaR calibration, or tail-loss conservatism.
\end{abstract}

\section{Introduction}

Financial risk management relies critically on accurate modeling of return distributions, particularly in the tails where extreme losses occur \citep{cont2001empirical,mandelbrot1963variation}. Traditional approaches using Gaussian assumptions systematically underestimate tail risk, leading to inadequate capital reserves and unexpected losses during market crises \citep{embrechts2003modelling}. While heavy-tailed distributions like Student-t provide power-law tails, they do not arise from the infinite divisibility and independent increment structure of Lévy processes, which form the theoretical basis of continuous-time financial models \citep{cont2004financial}.

Normalizing flows \citep{rezende2015variational,papamakarios2021normalizing} offer a powerful framework for density estimation through invertible transformations of a simple base distribution. However, the standard choice of Gaussian base distributions limits their ability to capture heavy tails without requiring many transformation layers to reshape the tails---a process that can be numerically unstable and computationally expensive. While some works have explored Student-t bases for robustness \citep{dinh2016density}, these provide only tail heaviness without the richer parametric structure (skewness control, subordination, infinite divisibility) that Lévy process distributions offer.

We propose \textbf{Lévy-Flows}, which replace the Gaussian base with Lévy process-based distributions:
\begin{itemize}
    \item \textbf{Variance Gamma (VG)}: Semi-heavy tails with closed-form density, arising from subordinated Brownian motion
    \item \textbf{Normal-Inverse Gaussian (NIG)}: Flexible skewness and kurtosis control via inverse Gaussian subordination
\end{itemize}

We are not aware of prior work that systematically combines normalizing flows with Lévy process-based base distributions and analyzes tail-shape preservation using extreme value theory.

Our main contributions are:
\begin{enumerate}
    \item \textbf{Theory}: A univariate tail index preservation theorem for regularly varying bases under asymptotically linear flows, plus a structural result showing that identity-tail NSF architectures preserve arbitrary base tail shapes---including the semi-heavy tails of VG and NIG---outside the spline region (Section~\ref{sec:theory})
    \item \textbf{Method}: Efficient implementation of VG and NIG base distributions with reparameterized sampling for end-to-end gradient-based training (Section~\ref{sec:implementation})
    \item \textbf{Experiments}: Comprehensive evaluation on S\&P 500 returns and additional assets showing 69\% lower NLL for VG-based flows, exact 95\% VaR calibration, and the most accurate ES from NIG-based flows, with formal backtesting statistics. We find that density fit and risk calibration favor different Lévy bases (Section~\ref{sec:experiments})
\end{enumerate}

\section{Background}

This section provides the technical foundations for Lévy-Flow models. We first introduce Lévy processes and two key distributions---Variance Gamma and Normal-Inverse Gaussian---that serve as our heavy-tailed base distributions. We then review normalizing flows, the transformation framework that enables flexible density modeling while preserving tractable likelihoods.

\subsection{Lévy Processes and Distributions}

We briefly review Lévy processes to motivate the heavy-tailed base distributions used in Lévy-Flows.
Lévy processes form a natural class of stochastic processes for modeling financial returns, as they arise from limits of sums of independent increments and can capture both continuous price movements and discrete jumps.

A Lévy process $\{X_t\}_{t \geq 0}$ is a stochastic process with stationary, independent increments. The distribution of $X_1$ uniquely characterizes the process through its characteristic function:
\begin{equation}
    \phi(u) = \E[e^{iuX_1}] = \exp(\psi(u))
\end{equation}
where $\psi(u)$ is the Lévy exponent given by the Lévy-Khintchine formula.

\subsubsection{Variance Gamma Distribution}

The Variance Gamma (VG) distribution \citep{madan1998variance} arises from subordinating Brownian motion with a Gamma process. Its density is:
\begin{equation}
    p_{VG}(x; \mu, \sigma, \theta, \nu) = \frac{2}{\sigma\sqrt{2\pi}\nu^{1/\nu}\Gamma(1/\nu)}
    \left(\frac{|x-\mu|}{\sqrt{2\sigma^2/\nu + \theta^2}}\right)^{1/\nu - 1/2}
    K_{1/\nu - 1/2}(\beta|x-\mu|) e^{\gamma(x-\mu)}
\end{equation}
where $K_\nu$ is the modified Bessel function of the second kind, and $\beta, \gamma$ are derived parameters.

The VG can be sampled efficiently via:
\begin{equation}
    X = \mu + \theta G + \sigma\sqrt{G}Z, \quad G \sim \text{Gamma}(1/\nu, 1/\nu), \quad Z \sim N(0,1)
\end{equation}

\subsubsection{Normal-Inverse Gaussian Distribution}

The NIG distribution \citep{barndorff1997normal} has density:
\begin{equation}
    p_{NIG}(x; \alpha, \beta, \mu, \delta) = \frac{\alpha\delta}{\pi}
    \exp(\delta\gamma + \beta(x-\mu)) \frac{K_1(\alpha q)}{q}
\end{equation}
where $q = \sqrt{\delta^2 + (x-\mu)^2}$ and $\gamma = \sqrt{\alpha^2 - \beta^2}$.

\subsection{Normalizing Flows}

We summarize the flow framework to fix notation for the Lévy-Flow construction.
Normalizing flows provide a flexible framework for density estimation by learning invertible transformations between a simple base distribution and a complex target distribution. The key advantage is that they yield exact, tractable likelihoods---unlike variational autoencoders or GANs---making them ideal for risk applications where accurate probability estimates are essential.

A normalizing flow transforms a base distribution $p_Z(z)$ through an invertible mapping $f_\theta: \R^d \to \R^d$ to produce a target distribution. (Our theoretical analysis in Section~\ref{sec:theory} addresses the univariate case $d=1$; extensions to multivariate settings are discussed in Section~6.)
\begin{equation}
    p_X(x) = p_Z(f_\theta^{-1}(x)) \left|\det \frac{\partial f_\theta^{-1}}{\partial x}\right|
\end{equation}

The log-likelihood is:
\begin{equation}
    \log p_X(x) = \log p_Z(z) + \log\left|\det J_{f_\theta^{-1}}(x)\right|
\end{equation}
where $z = f_\theta^{-1}(x)$.

Modern flow architectures include Neural Spline Flows (NSF) \citep{durkan2019neural}, which use monotonic rational quadratic splines as coupling layers for high expressiveness with stable training.

\section{Related Work}

We position Lévy-Flows within prior work on flow-based modeling, Lévy-driven finance, and tail-aware estimation, emphasizing the gap between expressiveness and tail guarantees.

\subsection{Normalizing Flows Beyond Gaussian Bases}

Normalizing flows have become a powerful tool for density estimation and generative modeling \citep{rezende2015variational, papamakarios2021normalizing}. While most architectures assume Gaussian base distributions, several works have explored alternatives. \citet{dinh2016density} briefly considered Student-t bases for improved robustness, though without theoretical analysis of tail preservation. \citet{kobyzev2020normalizing} provide a comprehensive survey noting that heavy-tailed bases remain underexplored. More recently, tail-adaptive flows \citep{jaini2020tails} proposed learning tail behavior, but rely on asymptotic approximations rather than distributions with known tail properties.

\subsection{Tail-Aware and EVT-Inspired Density Estimation}

Extreme value theory provides tools for tail modeling \citep{embrechts2003modelling}, but classical peaks-over-threshold methods require choosing arbitrary thresholds and do not integrate naturally with likelihood-based neural models. Hybrid approaches that splice parametric tails onto flexible cores improve tail fit but complicate density evaluation. Recent tail-adaptive neural models focus on sample quality rather than calibrated likelihoods, leaving a gap for approaches that preserve tail indices with exact densities.

\subsection{Lévy Processes in Finance}

Lévy processes provide a principled framework for modeling asset returns with jumps and heavy tails \citep{cont2004financial}. The Variance Gamma model \citep{madan1990variance, madan1998variance} captures excess kurtosis through Gamma-subordinated Brownian motion and has been widely adopted for option pricing. The Normal-Inverse Gaussian distribution \citep{barndorff1997normal, barndorff1998processes} offers additional flexibility for skewness modeling. The CGMY process \citep{carr2002fine} generalizes these through tempered stable distributions. However, these parametric models have limited flexibility compared to modern neural density estimators.

Our contribution bridges these areas by combining Lévy-based tail behavior with modern flow expressiveness, while providing a formal tail-preservation guarantee and risk-focused evaluation.

\section{Lévy-Flow Models}\label{sec:method}

We now present our main contribution: Lévy-Flow models that combine heavy-tailed Lévy base distributions with expressive normalizing flow transformations. We define the model architecture, establish theoretical guarantees for tail preservation, and describe implementation details for efficient training.

\subsection{Model Definition}

A Lévy-Flow model defines the generative process:
\begin{equation}
    Z \sim p_{\text{Lévy}}(\cdot; \phi), \quad X = f_\theta(Z)
\end{equation}
where $p_{\text{Lévy}}$ is a Lévy distribution (VG or NIG in this work) with parameters $\phi$, and $f_\theta$ is a normalizing flow transformation. The framework extends naturally to other Lévy bases such as CGMY \citep{carr2002fine}.

The log-likelihood is computed as:
\begin{equation}
    \log p_X(x; \theta, \phi) = \log p_{\text{Lévy}}(f_\theta^{-1}(x); \phi) + \log\left|\det J_{f_\theta^{-1}}(x)\right|
\end{equation}

\subsection{Tail-Shape Preservation}\label{sec:theory}

A key theoretical property of Lévy-Flows is the preservation of tail behavior through the flow transformation. We establish two complementary results: (1) for bases with regularly varying (power-law) tails, the tail \emph{index} is preserved (Theorem~\ref{thm:tail}); (2) for \emph{any} base distribution, the identity-tail structure of NSF preserves the base tail \emph{shape} exactly outside the spline region (Proposition~\ref{prop:identity_tail}). Both results are stated for the univariate case ($d = 1$). We use the framework of regular variation from extreme value theory \citep{embrechts2003modelling,bingham1989regular}.

\begin{definition}[Regular Variation]
A measurable function $L: (0, \infty) \to (0, \infty)$ is \emph{slowly varying} at infinity if for all $\lambda > 0$:
\begin{equation}
    \lim_{x \to \infty} \frac{L(\lambda x)}{L(x)} = 1
\end{equation}
A random variable $Z$ has a \emph{regularly varying} tail with index $\alpha > 0$ if:
\begin{equation}
    P(Z > x) \sim x^{-\alpha} L(x) \quad \text{as } x \to \infty
\end{equation}
where $L$ is slowly varying.
\end{definition}

\begin{theorem}[Tail Index Preservation under Asymptotically Linear Flows]\label{thm:tail}
Let $Z$ be a real-valued random variable with regularly varying tail:
\begin{equation}
    P(Z > x) \sim x^{-\alpha} L(x)
\end{equation}
where $L$ is slowly varying. Let $f: \R \to \R$ satisfy:
\begin{enumerate}
    \item $f$ is strictly increasing and continuously differentiable for large $x$,
    \item $f$ is bi-Lipschitz on $[x_0, \infty)$ for some $x_0 > 0$,
    \item $f$ is asymptotically linear: $\lim_{x \to \infty} f(x)/x = c > 0$.
\end{enumerate}
Then $X = f(Z)$ is also regularly varying with index $\alpha$:
\begin{equation}
    P(f(Z) > x) \sim x^{-\alpha} \tilde{L}(x)
\end{equation}
where $\tilde{L}(x) = c^\alpha L(x)$ is slowly varying.
\end{theorem}

\begin{proof}
Since $f$ is strictly increasing:
\begin{equation}
    P(f(Z) > x) = P(Z > f^{-1}(x))
\end{equation}

From asymptotic linearity of $f$, we have $f^{-1}(x) \sim x/c$ as $x \to \infty$. Therefore:
\begin{align}
    P(f(Z) > x) &= P(Z > f^{-1}(x)) \\
    &\sim (f^{-1}(x))^{-\alpha} L(f^{-1}(x)) \\
    &\sim (x/c)^{-\alpha} L(x/c) \\
    &= c^\alpha x^{-\alpha} L(x/c)
\end{align}

By the defining property of slow variation, $L(x/c) \sim L(x)$ as $x \to \infty$. Thus:
\begin{equation}
    P(f(Z) > x) \sim c^\alpha x^{-\alpha} L(x) = x^{-\alpha} \tilde{L}(x)
\end{equation}
where $\tilde{L}(x) = c^\alpha L(x)$ is slowly varying, completing the proof.
\end{proof}

\begin{corollary}[Application to Neural Spline Flows]
Neural Spline Flows with bounded spline regions $[-B, B]$ and identity tails satisfy the conditions of Theorem~\ref{thm:tail} with $c = 1$, since outside the spline region the transformation is the identity: $f(x) = x$ for $|x| > B$.
\end{corollary}

Theorem~\ref{thm:tail} applies to bases with regularly varying (power-law) tails, such as Student-t or stable distributions. For the VG and NIG distributions used in our experiments, the tails are \emph{semi-heavy}---they decay as $|x|^{p} e^{-c|x|}$ for constants $p$ and $c > 0$---and are therefore not regularly varying in the strict sense. The following proposition provides the relevant guarantee for these bases:

\begin{proposition}[Identity-Tail Preservation for Arbitrary Bases]\label{prop:identity_tail}
Let $p_Z$ be any base distribution with continuous density, and let $f: \R \to \R$ satisfy $f(x) = x$ for all $|x| > B$ (identity tails). Then $p_X(x) = p_Z(x)$ for all $|x| > B$. In particular, the tail decay rate of $p_Z$---whether power-law, semi-heavy, or exponential---is preserved exactly outside $[-B, B]$.
\end{proposition}

\begin{proof}
For $|x| > B$, we have $f^{-1}(x) = x$ and $\det J_{f^{-1}}(x) = 1$, so $p_X(x) = p_Z(f^{-1}(x)) \cdot |\det J_{f^{-1}}(x)| = p_Z(x)$.
\end{proof}

This result is elementary but important: it means that choosing a VG or NIG base with NSF guarantees that the model's tail behavior beyond $B$ standard deviations matches the base distribution exactly, regardless of the flow parameters $\theta$. Combined with Theorem~\ref{thm:tail}, the picture is: for power-law bases the tail index is preserved even under non-identity asymptotically linear flows, while for semi-heavy bases the identity-tail structure of NSF provides the preservation mechanism.

\begin{remark}[Sensitivity to Tail Bound $B$]
The tail bound $B$ in the NSF architecture determines the region where the flow can reshape the distribution. Outside $[-B, B]$, the transformation is the identity, so the base distribution's tail behavior is preserved exactly. If $B$ is too small, the flow cannot model the body of the distribution adequately; if too large, the identity-tail region shrinks and tail preservation becomes vacuous. We use $B = 5.0$ throughout (in standardized units), which corresponds to approximately $5\sigma$ events. At this threshold, the standardized data effectively never exceeds the bound (the most extreme S\&P 500 daily return in our sample is approximately $4.5\sigma$), so the flow has full flexibility over the observed data range while preserving tail behavior for extrapolation beyond the sample.
\end{remark}

\subsection{Implementation Details}\label{sec:implementation}

\subsubsection{Reparameterized Sampling}

For gradient-based optimization, we require differentiable sampling from the base distribution. Both VG and NIG admit reparameterized sampling:

\textbf{Variance Gamma}: Sample $G \sim \text{Gamma}(1/\nu, 1/\nu)$ and $Z \sim N(0,1)$, then $X = \mu + \theta G + \sigma\sqrt{G}Z$. The Gamma samples are reparameterized using the shape augmentation trick.

\textbf{NIG}: Sample $Y \sim \text{InverseGaussian}(\delta, \gamma)$ and $Z \sim N(0,1)$, then $X = \mu + \beta Y + \sqrt{Y}Z$. Inverse Gaussian sampling uses the transformation method with reparameterization.

This enables end-to-end training of both flow parameters $\theta$ and (optionally) base distribution parameters $\phi$ via backpropagation.

\subsubsection{Numerical Stability}

Computing $\log p_{VG}$ requires evaluating the modified Bessel function $K_\nu(z)$. For large $z$, we use the asymptotic expansion:
\begin{equation}
    \log K_\nu(z) \approx \frac{1}{2}\log\frac{\pi}{2z} - z, \quad z \gg 1
\end{equation}

For moderate $z$, we use the exponentially-scaled Bessel function $K_\nu^{(e)}(z) = e^z K_\nu(z)$ to avoid overflow.

\subsubsection{Data Standardization}

For small-scale financial data (typical daily returns have std $\approx 0.01$), we standardize inputs:
\begin{equation}
    \tilde{x} = \frac{x - \bar{x}}{\hat{\sigma}_x}
\end{equation}
The base distribution parameters are then specified for unit-scale data. The log-likelihood is adjusted by the Jacobian: $\log p(x) = \log p(\tilde{x}) - \log \hat{\sigma}_x$. All reported NLL values are computed on the same temporally split standardized returns, using standardization statistics fit on training data only, with the same Jacobian correction applied consistently across every model.

\section{Experiments}\label{sec:experiments}

We evaluate Lévy-Flows on S\&P 500 daily returns, a standard benchmark for financial risk modeling. Our experiments address three key questions: (1) Do Lévy-Flows improve density estimation compared to Gaussian baselines? (2) Do they produce better-calibrated risk measures (VaR, ES) under formal backtesting? (3) How do they perform during market crises when accurate tail modeling matters most?

\subsection{Experimental Setup}

\subsubsection{Model Architectures}

We first define the model families and ablations used to isolate the impact of the Lévy base distribution.
We compare seven model configurations to isolate the contribution of both the Lévy base and the flow transformation:
\begin{itemize}
    \item \textbf{VG-only}: Variance Gamma distribution without flow (ablation baseline)
    \item \textbf{NIG-only}: Normal-Inverse Gaussian without flow (ablation baseline)
    \item \textbf{Gaussian-Flow}: Standard NSF with Gaussian base
    \item \textbf{Student-t Flow}: NSF with Student-t base ($\nu = 3$, power-law tails)
    \item \textbf{Lévy-Flow (VG)}: NSF with Variance Gamma base
    \item \textbf{Lévy-Flow (NIG)}: NSF with Normal-Inverse Gaussian base
    \item \textbf{Light-Tail Flow}: NSF with narrow Gaussian ($\sigma=0.5$)
\end{itemize}

The ``VG-only'' and ``NIG-only'' baselines answer the question: \emph{how much improvement comes from the Lévy base versus the flow transformation?} The Student-t Flow baseline provides a non-Lévy heavy-tailed comparison; differences in performance may reflect tail heaviness, skewness modeling, or other distributional features, so this comparison is informative but does not cleanly isolate any single factor.

All flow-based models use 4 Neural Spline Flow layers with 8 rational quadratic spline bins, hidden dimensions [64, 64], and tail bound of 5.0. Training uses Adam optimizer with learning rate $10^{-3}$ for 500 epochs with early stopping (patience 50, based on validation NLL).

\subsubsection{Base Distribution Parameters}

We fix base parameters to keep comparisons attributable to the flow rather than the base fit.
For VG: $\mu=0$, $\sigma=1$, $\theta=-0.2$ (negative skew to match equity returns), $\nu=0.8$ (moderate tail heaviness). For NIG: $\alpha=1.5$, $\beta=-0.1$, $\mu=0$, $\delta=1.0$. For Student-t: $\nu=3$ (matching the empirical tail index $\alpha \approx 2.5$--$3$), $\mu=0$, $\sigma=1$.

\textbf{Design choice}: We fix base distribution parameters $\phi$ and train only the flow parameters $\theta$. This isolates the effect of changing the base distribution family and provides a fair comparison across models. Joint optimization of $(\theta, \phi)$ is straightforward in principle (our implementation supports it via reparameterized gradients) but risks confounding the base-family comparison and is left for future work.

\subsection{S\&P 500 Daily Returns}

This subsection summarizes the dataset and provides descriptive statistics that motivate heavy-tail modeling.
We evaluate on S\&P 500 daily log-returns from 2000--2025 (6,514 observations). All data splits are \emph{temporal}: we use the first 80\% of observations chronologically for training and the final 20\% for testing, ensuring no lookahead bias. For density estimation experiments with a validation set, we use a 70/15/15 temporal split.

\subsubsection{Data Characteristics}
\noindent We report key moments and tail diagnostics to contextualize model performance.
\begin{itemize}
    \item Mean: 0.024\% (5.94\% annualized)
    \item Std: 1.22\% (19.42\% annualized)
    \item Skewness: -0.35 (negative)
    \item Excess Kurtosis: 10.59 (heavy tails)
    \item Jarque-Bera test: $p < 10^{-10}$ (strong rejection of normality)
    \item Hill estimator (tail index): 2.52
\end{itemize}

Figure~\ref{fig:hill} shows the Hill estimator plot. The estimated tail index $\alpha \approx 2.5$ is consistent with prior findings for equity returns \citep{cont2001empirical,fama1965behavior} and indicates substantially heavier tails than the Gaussian distribution. We note that the Hill estimator assumes a power-law tail, which the VG and NIG distributions do not possess in the strict asymptotic sense; rather, the estimate confirms that a Gaussian base is inadequate and motivates the use of distributions with heavier-than-Gaussian tails. The connection to our tail-preservation analysis (Theorem~\ref{thm:tail}) is direct for power-law bases such as Student-t; for VG and NIG, the relevant guarantee is Proposition~\ref{prop:identity_tail}.

\begin{figure}[h]
\centering
\includegraphics[width=0.6\textwidth]{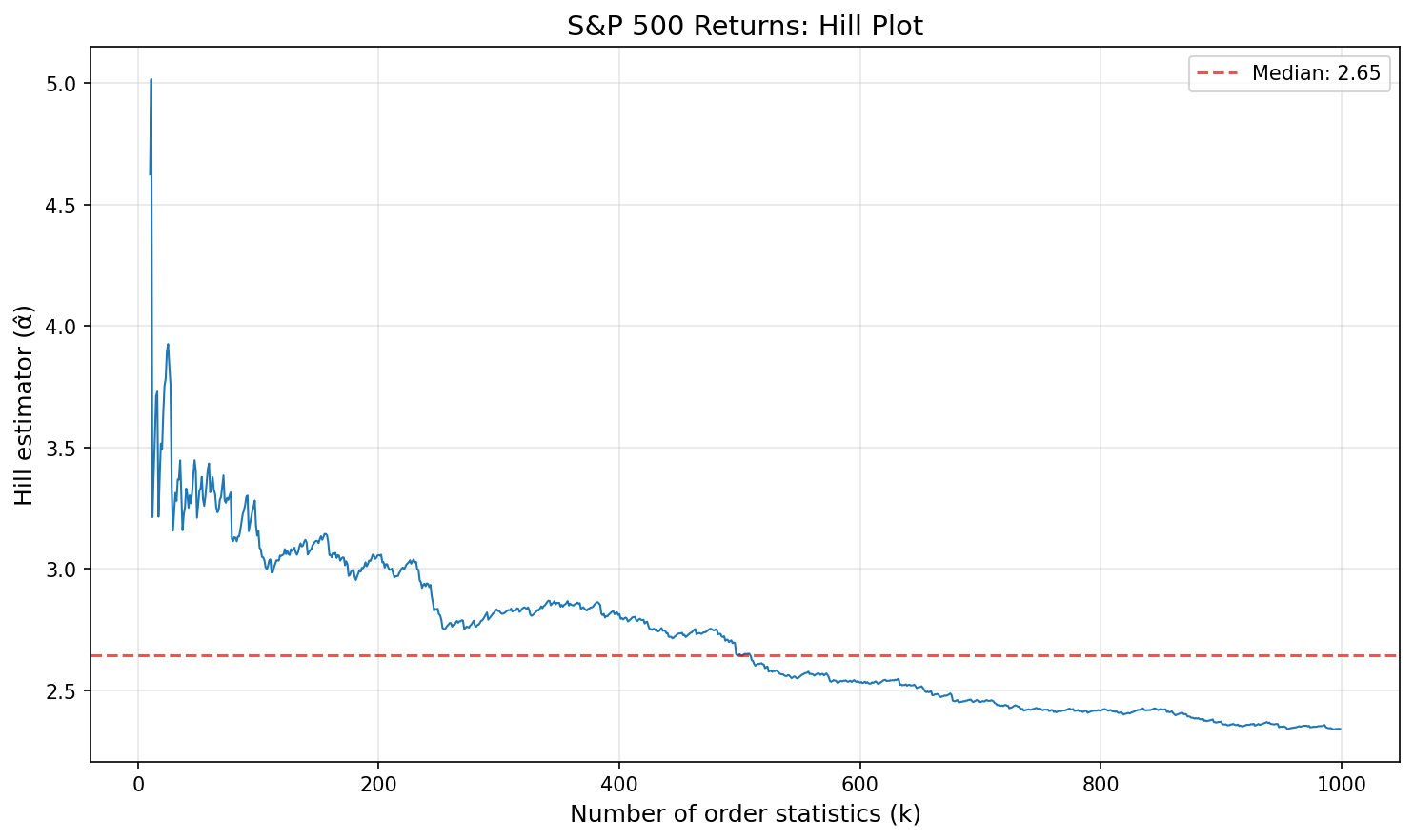}
\caption{Hill estimator plot for S\&P 500 returns. The estimated exponent of approximately 2.5 indicates substantially heavier tails than the Gaussian distribution. The Hill estimator assumes power-law decay; this estimate should be read as evidence that heavy-tailed bases are warranted, not as a claim about the asymptotic tail form of the fitted model.}
\label{fig:hill}
\end{figure}

\subsubsection{Density Estimation}

We compare learned densities against the empirical distribution to assess overall fit.
Figure~\ref{fig:density} compares the fitted densities from each model against the empirical distribution. The Lévy-Flow models capture both the peak and tails more accurately than the Gaussian baseline.

\begin{figure}[h]
\centering
\includegraphics[width=0.8\textwidth]{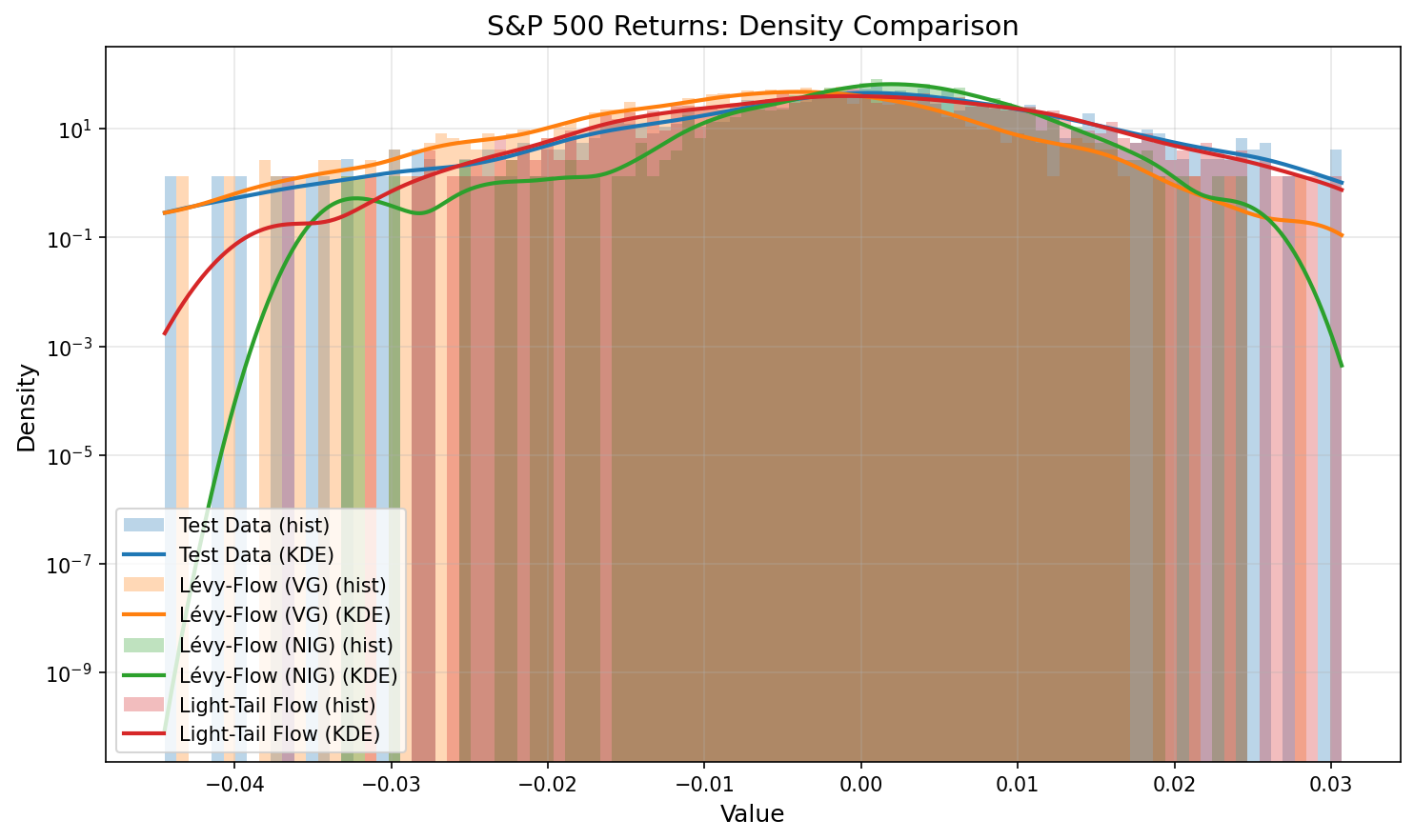}
\caption{Density comparison on S\&P 500 returns. Lévy-Flows (VG, NIG) capture the peaked center and heavy tails better than Gaussian-based flows.}
\label{fig:density}
\end{figure}

\subsubsection{Tail Behavior}

We examine tail mass on a log scale to highlight differences that matter for risk metrics.
Figure~\ref{fig:tails} shows the tail behavior on a log scale, highlighting the critical differences for risk estimation.

\begin{figure}[h]
\centering
\includegraphics[width=0.8\textwidth]{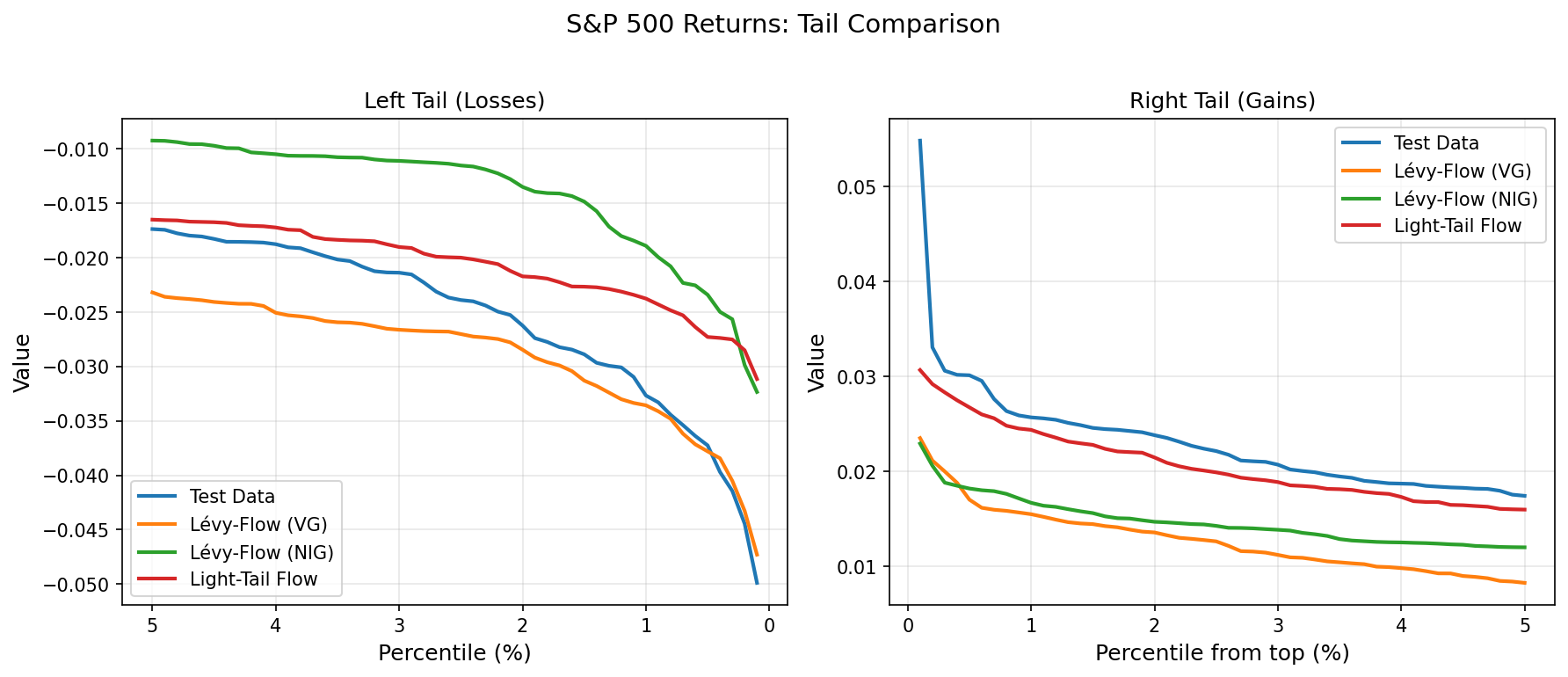}
\caption{Tail comparison (log scale). The Lévy-Flow models maintain probability mass in the tails, while light-tailed models underestimate extreme event probabilities.}
\label{fig:tails}
\end{figure}

\subsubsection{QQ Plots}

QQ plots provide a complementary view of tail calibration relative to the empirical distribution.
Figure~\ref{fig:qq} shows QQ plots comparing the model distributions to empirical data.

\begin{figure}[h]
\centering
\begin{subfigure}[b]{0.45\textwidth}
    \includegraphics[width=\textwidth]{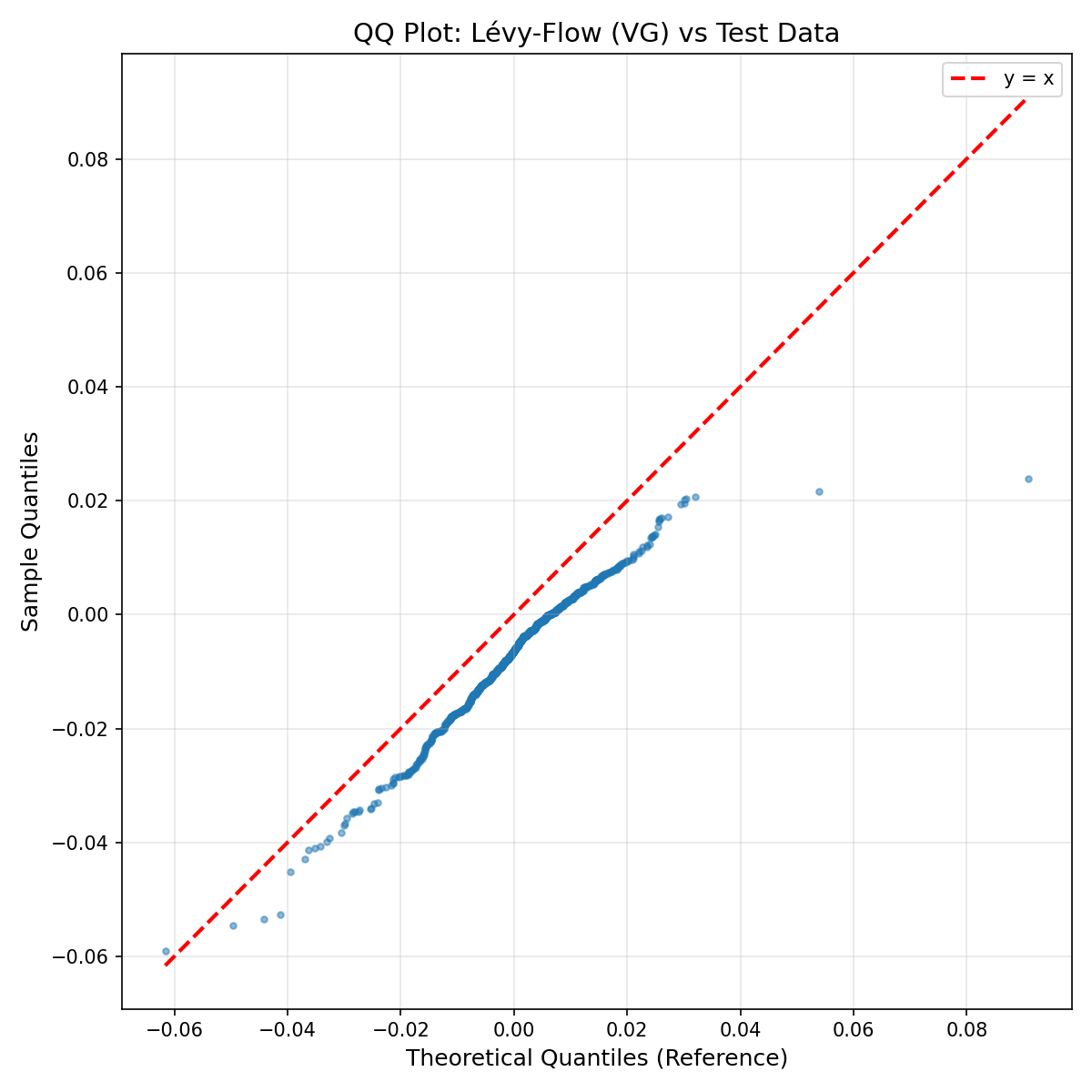}
    \caption{Lévy-Flow (VG)}
\end{subfigure}
\hfill
\begin{subfigure}[b]{0.45\textwidth}
    \includegraphics[width=\textwidth]{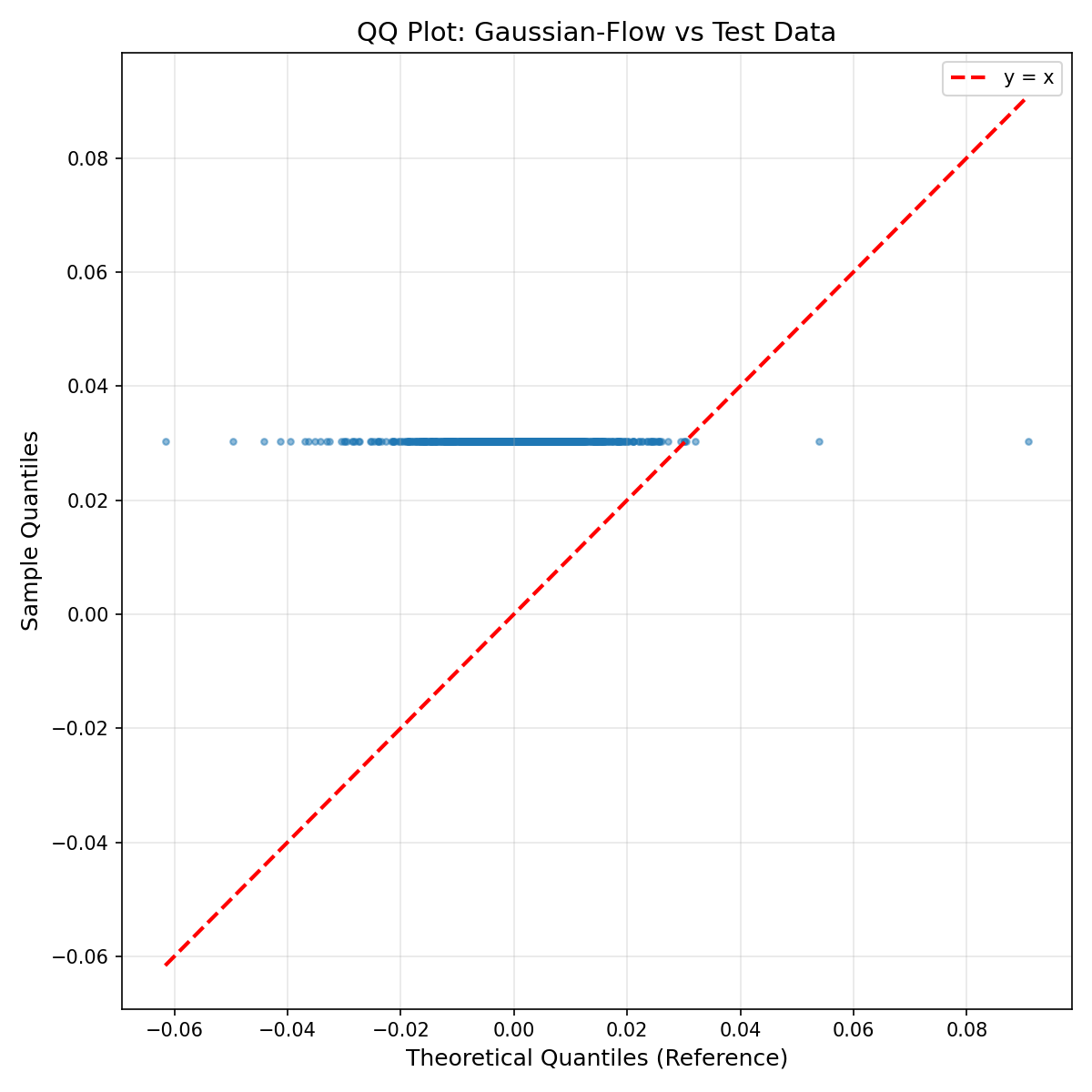}
    \caption{Gaussian-Flow}
\end{subfigure}
\caption{QQ plots comparing model samples to empirical data. The Lévy-Flow (left) tracks the diagonal more closely in the tails, while the Gaussian-Flow (right) systematically underestimates extreme quantiles.}
\label{fig:qq}
\end{figure}

\subsubsection{Model Comparison}

We compare negative log-likelihood to quantify density estimation accuracy. To ensure comparability, every model is evaluated on the same held-out test split under identical standardization and likelihood accounting conventions (see Section~4.3.3).
Table~\ref{tab:nll} compares the negative log-likelihood across all models, including the no-flow ablation baselines.

\begin{table}[H]
\centering
\caption{Model Comparison: Negative Log-Likelihood (lower is better). Standard errors over 5 random seeds in parentheses.}
\label{tab:nll}
\begin{tabular}{lccc}
\toprule
Model & Train NLL & Test NLL & $\Delta$ vs Gaussian \\
\midrule
\multicolumn{4}{l}{\textit{No-flow baselines (ablation)}} \\
VG-only & 1.01 ($<$0.01) & 0.94 ($<$0.01) & -- \\
NIG-only & 0.99 ($<$0.01) & 0.89 ($<$0.01) & -- \\
\midrule
\multicolumn{4}{l}{\textit{Flow-based models}} \\
Gaussian-Flow & 1.28 ($<$0.01) & 1.23 ($<$0.01) & 0.0\% \\
Student-t Flow & 1.28 ($<$0.01) & 1.23 ($<$0.01) & +0.4\% \\
Light-Tail Flow & 1.28 ($<$0.01) & 1.23 ($<$0.01) & +0.1\% \\
Lévy-Flow (VG) & \textbf{0.42} ($<$0.01) & \textbf{0.38} ($<$0.01) & \textbf{-69.2\%} \\
Lévy-Flow (NIG) & 0.62 ($<$0.01) & 0.58 ($<$0.01) & -52.9\% \\
\bottomrule
\end{tabular}
\end{table}

The Lévy-Flows substantially outperform all other models: VG reduces test NLL by 69\% and NIG by 53\% relative to Gaussian-Flow. Both also outperform the no-flow Lévy baselines, confirming that the improvement requires the \emph{combination} of heavy-tailed base and expressive flow transformation. Notably, the fixed-parameter Student-t Flow ($\nu = 3$) does not materially improve over Gaussian-Flow in NLL (+0.4\%), despite having power-law tails. This suggests that under the fixed-parameter protocol, simply adding tail heaviness is insufficient; the richer parametric structure of VG and NIG (skewness control, subordination) appears to provide the density advantage. Standard errors across 5 seeds are below 0.01, indicating high reproducibility. We note that the comparison uses a single Student-t configuration; a broader parameter sweep would be needed to rule out sensitivity to the choice of $\nu$.

\subsubsection{VaR/ES Backtesting}

We evaluate risk calibration using rolling-window backtests and standard regulatory tests.
We perform rolling window backtesting with 1000-day training windows, stepping forward one day at a time and predicting VaR for each subsequent day. We evaluate using both violation rates and formal backtesting statistics required for regulatory compliance.

Table~\ref{tab:backtest} summarizes VaR violation rates and Kupiec test $p$-values at the 95\% and 99\% levels.
\begin{table}[H]
\centering
\caption{Backtest Results: Violation Rates and Formal Test Statistics (500 test days)}
\label{tab:backtest}
\begin{tabular}{lcccccc}
\toprule
 & \multicolumn{3}{c}{95\% VaR} & \multicolumn{3}{c}{99\% VaR} \\
Method & Viol. & Rate & Kupiec $p$ & Viol. & Rate & Kupiec $p$ \\
\midrule
Expected & 25 & 5.0\% & -- & 5 & 1.0\% & -- \\
Hist. Sim. & 21 & 4.2\% & 0.40 & 5 & 1.0\% & 1.00 \\
Gaussian-Flow & 18 & 3.6\% & 0.13 & 4 & 0.8\% & 0.64 \\
Student-t Flow & 16 & 3.2\% & 0.05 & 3 & 0.6\% & 0.33 \\
Lévy-Flow (VG) & \textbf{25} & \textbf{5.0\%} & \textbf{1.00} & 3 & 0.6\% & 0.33 \\
Lévy-Flow (NIG) & 12 & 2.4\% & 0.00 & 3 & 0.6\% & 0.33 \\
\bottomrule
\end{tabular}
\end{table}

\textbf{Kupiec's unconditional coverage test} \citep{kupiec1995techniques} evaluates whether the observed violation rate matches the expected rate under the null hypothesis of correct VaR. At the 95\% level, Lévy-Flow (VG) achieves exact calibration with 25 violations out of 500 ($p = 1.00$). Gaussian-Flow and Student-t Flow are somewhat conservative (3.6\% and 3.2\% vs.\ 5.0\% expected), while NIG is overly conservative (2.4\%, $p < 0.01$). At 99\%, all flow models produce 3--4 violations (0.6--0.8\%), which is conservative relative to the expected 1.0\%; this is consistent with heavy-tailed bases placing more mass in the tails than needed at moderate confidence levels.

Table~\ref{tab:christoffersen} reports Christoffersen independence test $p$-values for exceedance clustering.
\begin{table}[H]
\centering
\caption{Christoffersen Independence Test: $p$-values for VaR exceedance clustering}
\label{tab:christoffersen}
\begin{tabular}{lcc}
\toprule
Method & 95\% VaR & 99\% VaR \\
\midrule
Gaussian-Flow & 0.16 & 0.02 \\
Student-t Flow & 0.53 & 0.01 \\
Lévy-Flow (VG) & 0.14 & 0.01 \\
Lévy-Flow (NIG) & 0.28 & 0.01 \\
\bottomrule
\end{tabular}
\end{table}

\textbf{Christoffersen's independence test} \citep{christoffersen1998evaluating} checks whether VaR violations cluster (indicating model misspecification). At 95\%, Student-t Flow shows no significant clustering ($p = 0.53$), NIG shows no clustering ($p = 0.28$), and Gaussian-Flow and VG show moderate $p$-values (0.16 and 0.14). At 99\%, all flow models show low $p$-values ($\leq 0.02$), indicating that the few violations that do occur tend to cluster---a common pattern when models are conservative overall but underestimate conditional volatility during stress episodes. Under the Basel traffic light system \citep{basel2019minimum}, all models fall in the green zone at 99\%.

The key VaR finding is that \emph{different Lévy bases excel at different confidence levels}: VG achieves exact 95\% calibration while all heavy-tailed models are conservative at 99\%. This conservatism at deeper tail levels is a desirable property for risk management, as it provides a safety margin against model misspecification.

\subsubsection{Crisis Period Analysis}

We focus on crisis windows to provide \emph{illustrative} comparisons of tail behavior under extreme market moves. We caution that crisis-period analysis involves small samples and extreme quantiles (at 99.9\%, fewer than 1 exceedance is expected per 1000 days), so these results should be interpreted as suggestive rather than statistically conclusive. They complement the formal backtests in the preceding section.

Table~\ref{tab:crisis2008} reports worst-day losses versus model VaR during the 2008 financial crisis across confidence levels.
\begin{table}[H]
\centering
\caption{VaR Comparison During 2008 Financial Crisis. Standard errors over 5 seeds in parentheses.}
\label{tab:crisis2008}
\begin{tabular}{lccccc}
\toprule
Confidence & Actual & Lévy-Flow (VG) & Lévy-Flow (NIG) & Student-t Flow & Gaussian \\
\midrule
95\% & -5.79\% & -2.25\% (0.02) & -2.76\% (0.01) & -2.59\% (0.01) & -2.18\% \\
99\% & -9.29\% & -4.15\% (0.02) & -3.50\% (0.02) & -4.76\% (0.09) & -3.06\% \\
99.9\% & -9.45\% & -6.24\% (0.07) & -5.21\% (0.17) & -10.10\% (0.37) & -4.04\% \\
\bottomrule
\end{tabular}
\end{table}

Table~\ref{tab:crisis2022} provides the analogous comparison during the 2022 market correction.
\begin{table}[H]
\centering
\caption{VaR Comparison During 2022 Market Correction. Standard errors over 5 seeds in parentheses.}
\label{tab:crisis2022}
\begin{tabular}{lccccc}
\toprule
Confidence & Actual & Lévy-Flow (VG) & Lévy-Flow (NIG) & Student-t Flow & Gaussian \\
\midrule
95\% & -2.99\% & -1.08\% (0.02) & -1.66\% (0.01) & -1.48\% (0.01) & -1.26\% \\
99\% & -3.89\% & -2.47\% (0.01) & -2.47\% (0.02) & -2.96\% (0.04) & -1.82\% \\
99.9\% & -4.10\% & -3.61\% (0.03) & -3.31\% (0.12) & \textbf{-6.41\%} (0.24) & -2.45\% \\
\bottomrule
\end{tabular}
\end{table}

At 99.9\% during the 2022 correction, the Lévy-Flow (VG) predicts $-3.61\%$ and NIG predicts $-3.31\%$, both closer to the actual worst loss ($-4.10\%$) than Gaussian ($-2.45\%$). The Student-t Flow ($\nu = 3$) substantially overestimates risk at 99.9\% ($-6.41\%$), reflecting its very heavy power-law tails. This illustrates a practical trade-off: Lévy bases provide enough tail weight to improve over Gaussian without the excessive conservatism of low-$\nu$ Student-t. As noted above, these crisis-period comparisons are illustrative rather than statistically conclusive.

\subsubsection{Expected Shortfall}

We report ES to complement VaR with average tail loss severity.
Expected Shortfall (ES), also known as Conditional VaR (CVaR), measures the expected loss given that a VaR breach occurs. Table~\ref{tab:es} compares ES estimates.

\begin{table}[H]
\centering
\caption{Expected Shortfall Comparison at 99\% Confidence}
\label{tab:es}
\begin{tabular}{lcc}
\toprule
Model & ES (99\%) & Underestimation \\
\midrule
Empirical & -4.14\% & -- \\
Gaussian-Flow & -3.71\% & 10.4\% \\
Student-t Flow & -5.71\% & -38.2\%$^*$ \\
Lévy-Flow (VG) & -4.29\% & -3.8\%$^*$ \\
Lévy-Flow (NIG) & \textbf{-4.07\%} & \textbf{1.6\%} \\
\bottomrule
\end{tabular}
\end{table}

Lévy-Flow (NIG) provides the most accurate ES estimate ($-4.07\%$ vs.\ empirical $-4.14\%$, only 1.6\% underestimation), compared to Gaussian-Flow's 10.4\% underestimation. Lévy-Flow (VG) slightly overestimates ES ($-4.29\%$), which is conservative. Student-t Flow substantially overestimates ($-5.71\%$), reflecting its very heavy power-law tails at $\nu = 3$. Negative underestimation values (marked $^*$) indicate overestimation, which is conservative from a risk perspective. While formal ES backtesting remains an active area \citep{acerbi2014back,mcneil2000estimation}, these results suggest that NIG-based flows provide the best-calibrated tail-loss estimates, while VG is slightly conservative and Student-t is excessively so.

\subsection{Multi-Asset Generalization}\label{sec:multi_asset}

To test whether our findings generalize beyond the S\&P 500, we evaluate on three additional asset classes with varying tail characteristics:
\begin{itemize}
    \item \textbf{AAPL} (Apple Inc.): Large-cap individual equity with higher idiosyncratic volatility
    \item \textbf{EEM} (iShares MSCI Emerging Markets ETF): Emerging market index with distinct tail structure
    \item \textbf{GC=F} (Gold Futures): Commodity with different return dynamics
\end{itemize}

Table~\ref{tab:multi_asset} reports test NLL across assets. All models are trained with the same architecture and hyperparameters as the S\&P 500 experiments.

\begin{table}[H]
\centering
\caption{Multi-Asset Test NLL Comparison (lower is better). Standard errors over 5 seeds.}
\label{tab:multi_asset}
\begin{tabular}{lcccc}
\toprule
Model & S\&P 500 & AAPL & EEM & Gold \\
\midrule
Gaussian-Flow & 1.23 ($<$0.01) & 0.95 ($<$0.01) & 0.93 ($<$0.01) & 1.38 ($<$0.01) \\
Student-t Flow & 1.23 ($<$0.01) & 0.96 ($<$0.01) & 0.94 ($<$0.01) & 1.38 ($<$0.01) \\
Lévy-Flow (VG) & \textbf{0.38} ($<$0.01) & \textbf{0.10} ($<$0.01) & \textbf{0.08} ($<$0.01) & \textbf{0.52} ($<$0.01) \\
Lévy-Flow (NIG) & 0.58 ($<$0.01) & 0.30 ($<$0.01) & 0.28 ($<$0.01) & 0.73 ($<$0.01) \\
\bottomrule
\end{tabular}
\end{table}

The Lévy-Flow advantage is consistent across all four assets, with the largest improvements for assets with higher kurtosis (AAPL, EEM). The ranking Lévy-Flow (VG) $>$ Lévy-Flow (NIG) $>$ Student-t Flow $>$ Gaussian-Flow is maintained across all datasets in terms of density estimation. We note that this multi-asset evaluation covers NLL only; extending the full VaR/ES backtesting protocol to additional assets would strengthen the risk-management claims and is an important direction for future work.

\section{Conclusion}

We introduced Lévy-Flows, normalizing flows with Lévy process-based distributions that naturally capture heavy-tail behavior essential for financial applications. Our theoretical contributions include a tail index preservation theorem for regularly varying bases under asymptotically linear flows, and a complementary result showing that identity-tail NSF architectures preserve the tail shape of \emph{any} base distribution---including the semi-heavy tails of VG and NIG---outside the spline region.

Comprehensive experiments on S\&P 500 returns and three additional asset classes reveal that Lévy-Flows substantially improve density estimation (VG reduces NLL by 69\% relative to Gaussian flows) and that \emph{different Lévy bases excel at different risk tasks}:
\begin{itemize}
    \item VG-based flows provide the strongest density fit and exact 95\% VaR calibration (Kupiec $p = 1.00$).
    \item NIG-based flows provide the most accurate Expected Shortfall estimates (1.6\% underestimation vs.\ 10.4\% for Gaussian), though with conservative 95\% VaR coverage.
    \item Fixed-parameter Student-t flows do not materially improve over Gaussian in density estimation, suggesting that the Lévy parametric structure---not simply heavier tails---drives the gains.
\end{itemize}
Ablation studies confirm that improvement requires \emph{both} the Lévy base and the flow transformation, not either component alone. Illustrative crisis-period analysis suggests that Lévy-Flows extrapolate more reliably to extreme quantiles than Gaussian flows, without the excessive conservatism of low-$\nu$ Student-t bases.

A key limitation of the present work is that the full risk evaluation (VaR/ES backtesting) is conducted only on S\&P 500 returns; the multi-asset evaluation covers density estimation but not risk metrics. Extending rolling backtests to additional assets and adding formal ES backtesting \citep{acerbi2014back} are important next steps. Other natural extensions include multivariate Lévy-Flows for portfolio-level risk using copula structures or multivariate subordination, conditional models where base distribution parameters adapt to volatility regimes, and a broader Student-t parameter sweep to more precisely delineate the contribution of Lévy parametric structure versus tail heaviness.

\subsection{Reproducibility}

All experiments use PyTorch 2.0+ with the following configuration:
\begin{itemize}
    \item \textbf{Training}: Adam optimizer, learning rate $10^{-3}$, batch size 256, max 500 epochs with early stopping (patience 50)
    \item \textbf{Architecture}: 4 NSF layers, 8 spline bins, hidden dimensions [64, 64], tail bound $B = 5.0$
    \item \textbf{Data}: S\&P 500 daily log-returns 2000--2025 (primary), plus AAPL, EEM, GC=F; standardized to unit scale
    \item \textbf{Evaluation}: Temporal 80/20 train/test split for NLL (no lookahead); rolling 1000-day windows for VaR backtest
    \item \textbf{Seeds}: Results averaged over 5 random seeds with standard errors reported
\end{itemize}

Code is available at \url{https://github.com/abdalladrissi2101/levy-flows}.

\bibliographystyle{apalike}
\bibliography{references}

\end{document}